\documentclass{article}
\usepackage{iclr2027_conference,times}

\usepackage{amsmath,amsfonts,bm}

\def\eqref#1{equation~\ref{#1}}

\def\1{\bm{1}}

\DeclareMathAlphabet{\mathsfit}{\encodingdefault}{\sfdefault}{m}{sl}
\SetMathAlphabet{\mathsfit}{bold}{\encodingdefault}{\sfdefault}{bx}{n}

\usepackage{hyperref}
\usepackage{url}

\usepackage{graphicx}
\usepackage{wrapfig}
\usepackage{algorithm}
\usepackage{algpseudocode}
\usepackage{booktabs}
\usepackage{multirow}
\usepackage{amssymb}
\usepackage{amsthm}
\newtheorem{lemma}{Lemma}

\theoremstyle{definition}

\title{Preconditioned Physics-Informed Neural Operator Training}

\author{Shizheng Wen\thanks{Equal contribution.} \\
ETH Zürich \\
\texttt{shizheng.wen@math.ethz.ch}
\And
Siddhartha Mishra \\
ETH Zürich \\
\texttt{siddhartha.mishra@math.ethz.ch}
\And
Marius Zeinhofer\footnotemark[1] \\
ETH Zürich \\
\texttt{marius.zeinhofer@math.ethz.ch}
}

\iclrfinalcopy 
\begin{document}

\maketitle
\lhead{Preprint}

\begin{abstract}
    Neural operators are typically trained in a supervised fashion, which requires a dataset to be generated with a classical solver. Training them physics-informed, i.e., purely from the governing equations, removes this large offline cost and allows fresh samples to be drawn at every optimization step, but has so far been limited to simplified problems and trails supervised training in accuracy. The obstacle is the ill-conditioning of physics-informed losses, which differential operators induce and which worsens as the discretization is refined. We therefore propose a preconditioned residual loss function and show mesh-independent conditioning for elliptic problems and greatly improved conditioning for saddle point problems. Realized through geometric and algebraic multigrid, the construction applies to linear and nonlinear equations, steady or time-dependent, on structured and unstructured meshes, is agnostic to the neural operator architecture, and adds no cost at inference. On the Poisson, Allen-Cahn and stationary Stokes equations, the resulting label-free training matches supervised training and is four to twenty-five times more accurate than previous physics-informed operator learning methods.
\end{abstract}

\section{Introduction}

\paragraph{Neural Operators}
Despite continued advances in numerical methods and computing power, solving partial differential equations (PDEs) remains computationally expensive. This is particularly relevant in many-query settings such as uncertainty quantification, optimal control, and inverse problems. The development of fast surrogate models for PDEs is therefore of great importance. With the advent of deep learning, neural operators have emerged as a powerful computational tool for this purpose \citep{li2021fourier, lu2021learning, kovachki2023neural,reno}, reducing PDE solution times by orders of magnitude compared to traditional solvers \citep{li2021fourier}. Neural operators approximate maps between infinite-dimensional function spaces by neural networks and are typically trained in a supervised fashion on a dataset generated offline by a numerical solver. Data generation is a major bottleneck, as it can constitute a substantial part of the overall computational and engineering effort \citep{NEURIPS2022_0a974713, NEURIPS2024_4f9a5acd}. This is especially the case for physics foundation models, which are pre-trained on a wide array of PDE types \citep{NEURIPS2024_84e1b1ec, mccabe2024multiple, mccabe2026walrus}, and in general, the speed-up at inference only pays off once this offline cost is amortized. Moreover, a finite dataset caps the attainable accuracy, as every sample is revisited many times during training.

\begin{figure}[t]
   \centering
   \includegraphics[width=\linewidth]{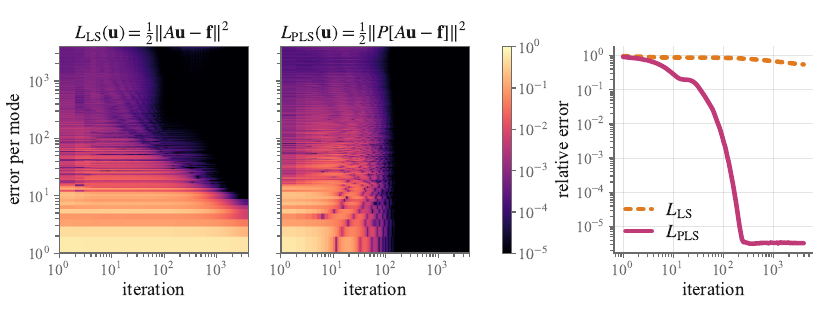}
   \caption{Ill-conditioning in physics-informed learning: The figure compares Adam minimization of a linear model with and without preconditioning for a physics-informed loss originating from a Poisson problem on a $65^2$ grid. The \textbf{left} and \textbf{centre} report error in each eigenmode of $A$ along the iteration, with the smoothest mode at the bottom. The \textbf{right} figure compares the relative $L^2$ error during training.}
   \label{fig:graphical_abstract}
\end{figure}

\paragraph{Physics-Informed Operator Training}
When the governing physics is exactly known, neural operators can---in principle---be trained in a unsupervised fashion, relying on residual or variational loss functions. This approach is referred to as \emph{physics-informed neural operator training} and has been pursued almost since the inception of neural operators \citep{li2024physics, wang2021learning}. It removes the limitations of supervised training discussed above: no dataset needs to be generated, samples can be drawn afresh at every training step---training at the infinite-data limit---and adapting to a new input distribution or geometry requires no new data. However, the success of physics-informed neural operator training has so far remained limited to simplistic problems. Purely physics-informed training converges slowly and to inaccurate solutions, as our experiments confirm (Table~\ref{tab:summary}), and adding training data is observed to ease the optimization \citep{li2024physics}. As a result, supervised learning with its offline data-generation pipeline remains the gold standard for almost all applications. This raises the question: what makes physics-informed losses so hard to train?

\paragraph{Ill-Conditioning of Physics-Informed Losses}
The fundamental reason is that differential operators amplify high frequencies. For example, the Laplacian scales a mode of frequency $\omega$ by $\omega^2$, so finer discretizations imply stronger amplification. To see how this affects training, consider a $P_1$ finite element discretization $A\mathbf u = \mathbf f$ of Poisson's equation with stiffness matrix $A$ and discrete solution $\mathbf u^* = A^{-1}\mathbf f$. On a single sample, the supervised and the physics-informed least-squares losses are
\begin{equation}\label{eq:data_vs_ls_loss}
   L_{\textup{data}}(\mathbf u) = \tfrac12\|\mathbf u - \mathbf u^*\|^2,
   \qquad
   L_{\textup{LS}}(\mathbf u) = \tfrac12\|A\mathbf u - \mathbf f\|^2 = \tfrac12\|A(\mathbf u - \mathbf u^*)\|^2.
\end{equation}
Both losses have the same minimizer and differ only through $A$. Yet this factor is decisive: the Hessian of $L_{\textup{data}}$ is the identity, whereas the Hessian of $L_{\textup{LS}}$ is $A^\top A$, whose condition number grows like $h^{-4}$ where $h$ is the mesh size. Since the number of gradient descent iterations scales with the condition number, physics-informed training becomes challenging already on coarse meshes. This is true even for a linear model without any neural network and optimization with Adam \citep{kingma2015adam}, see Figure~\ref{fig:graphical_abstract}.

\paragraph{Preconditioned Physics-Informed Training}
To remove this ill-conditioning, we use \emph{preconditioning}. We modify the physics-informed loss to
\begin{equation}\label{eq:preconditioned_least_squares_loss}
   L_{\textup{PLS}}(\mathbf u) = \tfrac12\|P(A\mathbf u - \mathbf f)\|^2 = \tfrac12\|PA(\mathbf u - \mathbf u^*)\|^2,
\end{equation}
with an invertible preconditioner $P$. The choice $P = I$ recovers $L_{\textup{LS}}$, and $P = A^{-1}$ recovers the perfectly conditioned supervised loss $L_{\textup{data}}$. Between these extremes, fast approximate inverses $P\approx A^{-1}$---we rely on geometric and algebraic multigrid methods \citep{brenner2008mathematical}---bring the conditioning close to that of supervised training and add negligible overhead. Both the residual and the preconditioner are assembled and applied on the GPU inside the training loop with a differentiable finite element stack \citep{naumov2015amgx,wen2026learning, chi2026torchsla}. Since $P$ enters only the loss function, the approach works with any neural operator architecture and adds no extra cost at inference. The construction extends to nonlinear residuals, where $P$ approximates the inverse Jacobian, and to saddle-point problems via block-diagonal preconditioners in Section~\ref{sec:stokes}.

\subsection{Main Contributions} The main contributions are
\begin{itemize}
   \item \textbf{A preconditioned physics-informed loss.} We propose a novel, preconditioned physics-informed loss function, reducing physics-induced ill-conditioning from $\mathcal O(h^{-4})$ to $\mathcal O(1)$ for elliptic problems and to $\mathcal O(h^{-2})$ for saddle point problems. The construction applies to linear and nonlinear equations, is agnostic to the neural operator architecture, adds no cost at inference, and, realized through multigrid methods, adds negligible cost during training.
   \item \textbf{Physics-informed training that matches supervised training.} This lets us train neural operators fully physics-informed on the Poisson, Allen--Cahn and stationary Stokes equations for both structured and unstructured grids. For all examples, we closely match or surpass supervised training (Table~\ref{tab:summary}). The PINO and physics-informed DeepONet baselines are four to twenty-five times less accurate under the identical protocol. Since no labels are needed, samples can moreover be drawn afresh at every step; for the Poisson equation, this matches supervised training on $4096$ labelled solutions (Figure~\ref{fig:data-scaling}).
\end{itemize}

\subsection{Related Literature}
The challenging training of physics-informed losses is well documented for neural PDE solvers that solve a single PDE instance, such as PINNs \citep{raissi2019physics} or the deep Ritz method \citep{e2018deep}. Different loss terms converge at different rates \citep{wang2021understanding} and need adaptive re-weighting \citep{wang2022and}. Similar to our viewpoint, \citet{de2024operator} attribute ill-conditioning to the underlying PDE operator. The most effective remedies are geometric optimization methods \citep{muller2023achieving, muller2024optimization} that precondition in parameter space based on the problem structure in function space. However, they scale like $\mathcal O(D^3)$ in the number of network parameters, and the low-rank variants that avoid this \citep{guzman2025improving} replace it by $\mathcal O(N^2D)$ in the number of residual evaluations $N$. For neural operators, both are out of reach. An alternative route considers preconditioners as part of the architecture for such neural PDE solvers, we discuss this in detail in Appendix~\ref{appendix:preconditioning-through-architecture}.

Physics-informed training of neural operators was introduced in \citep{li2024physics} and physics-informed DeepONets \citep{wang2021learning}, which minimize strong-form residuals obtained by finite differences and by automatic differentiation. Remedies for their optimization are loss reweighting, collocation sampling and architecture choices carried over from PINNs \citep{chen2026training}, as well as adding labelled data \citep{li2024physics} or the use in finetuning \citep{medvedev2026physics}. Finite element residuals have also been used as operator-learning losses \citep{yamazaki2025finite, wen2026learning}. None of these works addresses the ill-conditioning of PDE operators in the loss, which is the decisive factor.

\section{Methods}\label{sec:physics_informed_training}

\paragraph{Problem Setting}
We aim to approximate an operator $G$ between function spaces $\Pi$ and $U$
\begin{equation*}
   G: \Pi \to U, \quad \rho \mapsto G(\rho)
\end{equation*}
and we set $u(\rho)= G(\rho)$. Typically, we endow $\Pi$ with a probability measure $\mu$ corresponding to the data distribution the application dictates. The operator $G$ is approximated by a neural operator $G_\theta$---a deep neural network with trainable weights $\theta\in\Theta=\mathbb R^D$ that maps between the same spaces---and we set $u(\theta, \rho)=G_\theta(\rho)$. In our applications, $G$ is implicitly given as the solution operator of a parametric PDE problem, meaning $u(\rho)$ solves
\begin{equation}\label{eq:residual_characterization_of_neural_operator}
   R(u(\rho); \rho) = 0,
\end{equation}
where $R:U\times \Pi \to V^*$ is the PDE residual, $V$ is a space of test functions and $V^*$ its dual. PDE problems in the form of \eqref{eq:residual_characterization_of_neural_operator} are frequently solved by Galerkin methods, restricting the space $U$ to a subspace $U_h$ spanned by basis functions $\phi_1,\dots,\phi_n$ and by replacing the co-domain by the dual $V_h^*$ of a finite-dimensional space of test functions $V_h\subset V$. Solving the discretized PDE for a given $\rho\in\Pi$ then means to find $u_h(\rho)\in U_h$ satisfying $R_h(u_h(\rho);\rho)=0$. 

\paragraph{Example} For the solution operator of the homogeneous Poisson equation on the domain $\Omega\subset \mathbb R^d$, we set $U = V = H^1_0(\Omega)$ and $\Pi = V^* = H^{-1}(\Omega)$, and the residual is defined as
\begin{equation*}
   R(u,\rho)(v) = \int_\Omega\nabla u\nabla v\,\mathrm dx - \rho(v).
\end{equation*}
After a choice of Galerkin space $U_h= \operatorname{span}\{\phi_1, \dots, \phi_n\}$ and test space $V_h=U_h$, so that the co-domain of the discrete residual is $V_h^*$, which we identify with $\mathbb R^n$ through the basis, the discrete residual is
\begin{equation*}
   R_h(\textbf u, \rho) = A \mathbf u - \mathbf{f}, 
   \quad 
   A_{ij}=\int_\Omega \nabla \phi_j\nabla \phi_i\,\mathrm dx
   \quad
   \mathbf f_i = \rho(\phi_i),
   \quad
   i,j=1,\dots,n.
\end{equation*}

\paragraph{Interpolated Neural Operators} 
To use the discretized residual equation in the context of neural operators, we use an \emph{interpolation operator} $I_h$ mapping a neural operator output into $U_h$. This results in what we call an \emph{interpolated neural operator}
\begin{equation}\label{eq:interpolated_neural_operator}
   G_{\theta, h} = I_h \circ G_\theta: \Pi \to U_h, 
   \quad
   \rho \mapsto I_h\left( G_\theta(\rho) \right) = \sum_{i=1}^n  \mathbf u(\theta, \rho)_i \phi_i.
\end{equation}
The choice of the space $U_h$ and the interpolation operator $I_h$ should be seen as a modeling choice that depends on the concrete PDE problem at hand. We use Lagrange interpolation operators and Lagrange finite element space in our applications; the interpolation is part of the loss rather than of the model, and ties neither the operator to a fixed resolution nor its accuracy beyond the interpolation error, see Appendix~\ref{appendix:interpolated_neural_operators_details}.

\subsection{Physics-Informed Loss Functions}

\paragraph{Least Squares Loss}
The discrete residual $R_h$ and the interpolated neural operator $G_{\theta,h}$ directly yield a physics-informed loss 
\begin{equation}\label{eq:general_ls_loss}
   L_{\textup{LS}}(\theta)
   =
   \mathbb E_{\rho\sim\mu}\left[ \frac12 \left\| R_h\big(\mathbf u(\theta,\rho);\rho\big) \right\|^2 \right].
\end{equation}
Minimization of $L_{\textup{LS}}$ is achieved when $G_{\theta, h}(\rho)$ equals the Galerkin solution $u_h(\rho)$. In this form however, $L_{\textup{LS}}$ is ill-conditioned and difficult to train. PINO \citep{li2024physics} losses are of a similar form, and while physics-informed DeepONet \citep{wang2021learning} is based on an coordinate-based formulation, both share the ill-conditioning with this formulation.

\paragraph{Preconditioned Least Squares Loss}
Let $P$ be an invertible matrix, and $B$ be a positive definite and symmetric matrix. We define the \emph{preconditioned least-squares loss} as
\begin{equation}\label{eq:general_pls_loss}
   L_{\textup{PLS}}(\theta)
   =
   \mathbb E_{\rho\sim\mu}\left[ \frac12 \left\| P\,R_h\big(\mathbf u(\theta,\rho);\rho\big) \right\|_B^2 \right],
   \qquad \|r\|_B^2 = r^\top B\, r.
\end{equation}
Since $P$ is invertible and $B$ is positive definite, the integrand of \eqref{eq:general_pls_loss} vanishes exactly where that of \eqref{eq:general_ls_loss} does. Hence \emph{preconditioning changes the geometry of the loss landscape, but not its minimizers}. We stress again that our construction is architecture agnostic, and adds no additional cost at inference, as it only concerns the loss function.

As $\|Pr\|_B = \|r \|_{P^TBP}$, the split into the pair $(P,B)$ is redundant. We keep it, as it is a convenient way to specify our choices. Let us illustrate the case where $R_h(\mathbf u(\theta, \rho);\rho)= A\mathbf u(\theta, \rho) - \mathbf f(\rho)$ for a positive definite and symmetric matrix $A$, and assume the condition number of $A$ scales like $h^{-2}$, as is the case for Poisson's equation, a longer discussion with complete proofs for the conditioning claims is given in Appendix \ref{appendix:least-squares-preconditioning}. Consider the following choices
\begin{itemize}
   \item $(P,B) = (I,I)$. This recovers $L_{\textup{LS}}$ and induces $\mathcal O(h^{-4})$ scaling of the condition number.
   \item $(P,B) \approx (A^{-1}, M)$, where $M$ is of conditioning $\mathcal O(1)$, as is the case for the identity or the mass matrix. In this case, the overall conditioning is $\mathcal O(1)$ and we recover data-driven training in the limit of $P=A^{-1}$. 
   \item $(P,B) \approx (I, A^{-1})$. This requires $A$ to be symmetric and positive definite and recovers energy formulations. The conditioning is $\mathcal O(h^{-2})$ and the details are given in the appendix.
\end{itemize}

\section{Results}\label{sec:numerical_results}

We evaluate preconditioned physics-informed training on three operators: the solution operator of Poisson's equation on a structured grid, a time-stepping operator of the Allen--Cahn equation on a structured grid, both with an FNO, and stationary Stokes flow on an unstructured mesh, with GAOT \citep{gaot}. On the structured grids the preconditioner is geometric multigrid; on the unstructured mesh a block-diagonal preconditioner relying on algebraic multigrid. We compare against supervised training, PINO \citep{li2024physics} and physics-informed DeepONet \citep{wang2021learning}. The precise training protocol is provided in Appendix~\ref{appendix:details_poisson}. Table~\ref{tab:summary} summarizes our findings. Preconditioned physics-informed training matches supervised training on Poisson, trails it slightly on Allen--Cahn, and surpasses it on Stokes, at an overhead of under $10\%$ to about a third of the time per epoch. The baselines are far less accurate on Poisson and Stokes. On Allen--Cahn the picture differs: PINO comes within a factor of two of our loss, while physics-informed DeepONet fails altogether at $90\%$ error. The better performance of PINO is expected here, since the time-stepping operator is far better conditioned than a stationary elliptic problem. Nevertheless, the preconditioner leads to faster and more stable convergence, as can be seen in Figure \ref{fig:ac-results}. 

\begin{table}[t]
  \caption{We summarize our main results across the three PDE instances tested: Poisson, Allen--Cahn and Stokes equation. We report mean relative $L^2$ errors on the test set in percent over three seeds with half-range. For Stokes we report velocity and pressure error separately. Poisson and Allen--Cahn use a uniform grid with an FNO; for Stokes we employ an unstructured grid on a domain with a circular hole and use the GAOT operator. The last column is the wall-clock time of one training epoch on a single RTX 4090. Bold marks
  the most accurate physics-informed approach; $L_\textup{data}$ requires labelled solutions and is
  included as the supervised reference.}
  \label{tab:summary}
  \centering
  \footnotesize
  \setlength{\tabcolsep}{4.5pt}
  \begin{tabular}{llccccr}
    \toprule
    \multirow{2}{*}{PDE} & \multirow{2}{*}{loss} & physics- & \multirow{2}{*}{derivatives}
      & \multicolumn{2}{c}{rel.\ test $L^2$ [\%]} & time \\
    \cmidrule(lr){5-6}
      & & informed & & $u$ & $p$ & [s] \\
    \midrule
    \multirow{4}{*}{Poisson}
      & $L_\textup{PLS}$   & yes & FEM & $\mathbf{5.1 \pm 0.8}$  & ---            & 1.2 \\
      & $L_\textup{data}$  & no  & --- & $4.7 \pm 0.8$  & ---            & 1.1 \\
      & PINO               & yes & FD  & $33.1 \pm 1.1$ & ---            & 1.1 \\
      & PI-DeepONet        & yes & AD  & $21.4 \pm 4.7$ & ---            & 3.1 \\
    \midrule
    \multirow{4}{*}{Allen--Cahn}
      & $L_\textup{PLS}$   & yes & FEM & $\mathbf{3.2 \pm 0.3}$  & ---            & 46 \\
      & $L_\textup{data}$  & no  & --- & $2.0 \pm 0.2$  & ---            & 40 \\
      & PINO               & yes & FD  & $5.3 \pm 1.6$  & ---            & 40 \\
      & PI-DeepONet        & yes & AD  & $90.0 \pm 1.7$ & ---            & 8.2 \\
    \midrule
    \multirow{3}{*}{Stokes}
      & $L_\textup{PLS}$   & yes & FEM & $\mathbf{1.6 \pm 0.2}$  & $\mathbf{0.85 \pm 0.04}$ & 3.9 \\
      & $L_\textup{data}$  & no  & --- & $3.3 \pm 0.2$  & $1.79 \pm 0.16$ & 2.9 \\
      & PI-DeepONet        & yes & AD  & $39.7 \pm 2.1$ & $36.7 \pm 0.6$ & 6.0 \\
    \bottomrule
  \end{tabular}
\end{table}

\subsection{Poisson Equation}\label{sec:poisson}
We approximate the solution operator $G: \rho \mapsto u$ of the homogeneous Poisson equation
\begin{align}\label{eq:poisson_equation}
      -\Delta u &= \rho \quad \text{in }\Omega,
      \quad
      u = 0 \quad \text{on }\partial\Omega,
\end{align}
on the domain $\Omega = [0,1]^2$ and the forces $\rho$ are randomly drawn truncated sine series with a parameter $K$ that controls the number of active modes. The benchmarks of Table~\ref{tab:summary} are produced with $K=10$. We discretize in the space $U_h = Q_1^h$ of bilinear finite elements on a uniform $65\times 65$ grid which results in the discrete residual $R_h(\mathbf u, \rho)= A\mathbf u - \mathbf f$ with $A$ being the stiffness matrix and $\mathbf f$ the discrete right-hand side, see Section~\ref{sec:physics_informed_training} for details. The neural operator is an FNO whose output grid coincides with the finite element nodes. As preconditioner we take a single dyadic geometric multigrid V-cycle for $A$, the approximate inverse setting $(P,B)\approx (A^{-1}, I)$. Appendix~\ref{appendix:details_poisson} collects the full experiment protocol, below we report additional experiments besides the benchmark results in Table~\ref{tab:summary}.

\paragraph{Ill-Conditioning Prevents Physics-Informed Training}
In this experiment, we show that the ill-conditioning of the stiffness matrix $A$ is the reason physics-informed training fails without a preconditioner. To demonstrate this, we use the physics-informed least-squares loss with $B=I$ and a parametrized preconditioner
\begin{equation}
   L^t_\textup{PLS}(\mathbf u) = \frac12 \|P_t[A\mathbf u - \mathbf f] \|^2_2, \quad P_t = (1-t)I + t A^{-1}.
\end{equation}

The preconditioner $P_t$ is a convex mixing between no preconditioning $P_0=I$ and data-driven training $P_1 = A^{-1}$. This choice is possible on a small $65\times 65$ grid, and serves to illustrate the effect of preconditioning on the training dynamics. We monitor the relative $L^2(\Omega)$ error averaged over the validation set during training for a number of $t$-values, and compare to the single dyadic V-cycle geometric multigrid preconditioner. The results are reported in Figure~\ref{fig:pls-overlay}, where we additionally report the condition number of the Hessian $H_t=\nabla_u^2 L_\textup{PLS}^t = AP_t^2A$. For this experiment, the dataset uses $K=4$.

\begin{figure}[t]
  \centering
  \includegraphics[width=0.49\linewidth]{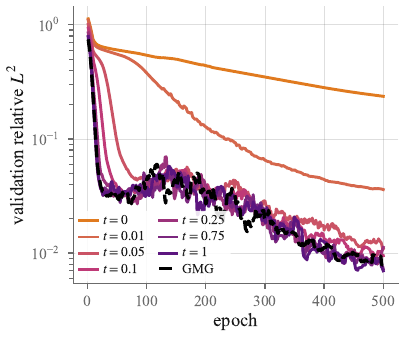}
  \hfill
  \includegraphics[width=0.49\linewidth]{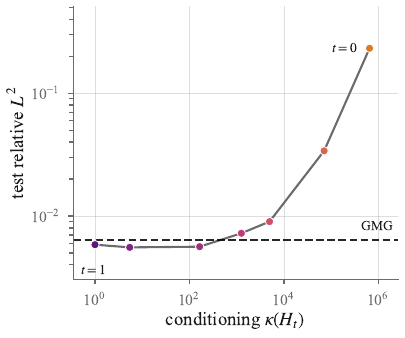}
  \caption{Preconditioned least-squares training for the Poisson equation. \textbf{Left:} validation relative $L^2$ vs.\ epoch as the preconditioner $P_t$ is swept from $I$ to $A^{-1}$, together with the geometric multigrid preconditioner \textbf{Right:} test error, at the best-validation checkpoint, of the $P_t$ preconditioned sweeps against the condition number $\kappa(H_t)$ of the preconditioned loss Hessian; the dashed line is the geometric multigrid preconditioner, which has no finite $\kappa(H_t)$.}
  \label{fig:pls-overlay}
\end{figure}

\begin{wrapfigure}{r}{0.40\textwidth}
   \vspace{-\intextsep}
   \includegraphics[width=\linewidth]{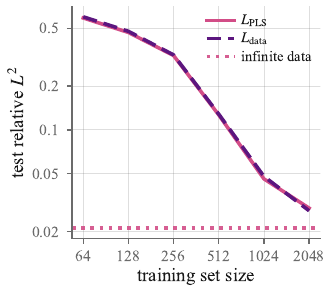}
   \vspace{-1em}
   \caption{Test error vs.\ training set size at a fixed optimization budget for finite and infinite datasets.}
   \label{fig:data-scaling}
\end{wrapfigure}
Figure~\ref{fig:pls-overlay} shows that preconditioning is essential to obtain efficient physics-informed training. The un-preconditioned run ($t=0$) converges to $\approx 30\%$ relative $L^2$ error after 500 epochs of training with Adam, while data-driven training ($t=1$) lands at below $1\%$ error for the same amount of training epochs. Intermediate values of $t$ interpolate the training dynamics, which leads us to conclude that ill-conditioning prevents efficient physics-informed neural operator training, and preconditioning is an effective strategy to mitigate this. This is further supported by tracking the condition numbers of $H_t$ in dependence on $t$. For $t=0$, we have $\kappa(H_t)=6.5e+05$ and of course for data-driven training it holds $\kappa(H_1)=1$. Interestingly, as visualized on the right of Figure \ref{fig:pls-overlay} below a conditioning of $\approx 10^3$--$10^4$ (corresponding to $t\in[0.05, 0.1]$) we observe diminishing returns for making the preconditioners more accurate. We attribute this to ill-conditioning induced by the neural network parametrization, which is not resolved in our analysis. This suggests, that a highly-accurate preconditioner may not be needed for accurate physics-informed training, and in fact, the geometric multigrid preconditioner performs virtually identical to data-driven training.

\paragraph{Training at The Infinite Data Limit} Physics-informed training is not bound to a dataset: we may draw a fresh batch of samples for every parameter update, so that no sample is ever revisited. We refer to this as training at the infinite data limit. To illustrate the benefit of training in that regime, we fix the optimization budget at $32{,}000$ updates at batch size $32$ and vary only the size of the training set, for both $L_\textup{PLS}$ and supervised training on the $K=10$ data set, see Figure~\ref{fig:data-scaling}. At every finite size up to the maximum dataset size considered at $2048$ the two are indistinguishable. This means at this budget the accuracy is limited by the amount of data, not by the choice of loss. The infinite dataset approach yields the best result at $2.1\%$ relative $L^2$ error versus the $2.8\%$ error of the data-driven approach on the largest dataset considered.

\subsection{Allen-Cahn Equation}\label{sec:allen_cahn}
We now learn a time-stepping operator of a convex-concave splitting scheme \citep{eyre1998unconditionally} for the Allen-Cahn equation with homogeneous Dirichlet boundary conditions
\begin{align}\label{eq:allen_cahn}
   \begin{split}
      \partial_t u = \Delta u + \varepsilon^2 u(1-u^2), \quad u_{|\partial\Omega}=0, \quad u_{|t=0} = u_0
   \end{split}
\end{align}
on the domain $\Omega=[0,1]^2$ and time interval $I=[0,1]$ with $\varepsilon=32$. That is for a time step $\tau$ and given $\rho=u_k$ the operator $G(u_k)=u_{k+1}$ satisfies
\begin{equation}
   \tfrac1\tau (u_{k+1} - u_k, v) + (\nabla u_{k+1}, \nabla v) + \varepsilon^2 (u_{k+1}^3 - u_k, v) = 0, \quad \text{for all }v\in H^1_0(\Omega)
\end{equation}
where we use $\tau=0.01$ and the sine-series dataset with $K=4$ as initial conditions. As neural operator, we use an FNO and for the spatial discretization we use a regular $129\times 129$ grid, hence the finite element space $Q_1^h$ as in the case of Poisson's equation which leads to the discrete residual ($\rho=u_k$)
\begin{equation}\label{eq:convex_concave_time_stepper}
   0 
   =
   M\left[ \frac{\mathbf u_{k+1} - \mathbf u_k}{\tau} + \varepsilon^2\left( \mathbf u_{k+1}^3 - \mathbf u_k \right) \right]
   +
   A \mathbf u_{k+1}
   =:
   R_h(\mathbf u_{k+1}, \mathbf u_k), 
\end{equation}
where $A_{ij}=\int_\Omega \nabla \phi_j\nabla \phi_i\,\mathrm dx$ is the stiffness matrix, and $M=\int_\Omega \phi_j \phi_i\,\mathrm dx$ is the mass matrix. We train the interpolated operator $G_{\theta, h}$ on 10 timesteps between $t=0$ and $t=0.1$. In Table~\ref{tab:summary}, we report the test error at $t=0.1$ and in Figure~\ref{fig:ac-results} we investigate roll-out stability. The problem is nonlinear and we discuss preconditioning below; the full details are given in Appendix~\ref{appendix:details_allen_cahn}.

\paragraph{Preconditioned Autoregressive Loss Functions}

Given an initial condition $\mathbf u_0$, we denote the reference roll-out by $\mathbf u_0, \mathbf u_1, \dots, \mathbf u_T$ which is produced from \eqref{eq:convex_concave_time_stepper} with a finite element method. The roll-out of the FNO is denoted by $G_{\theta,h}^i(\mathbf u_0)$, and when we write $\hat{\mathbf u}_i$, we mean $G_{\theta,h}^i(\mathbf u_0)$ with \emph{detached gradients}. The data-driven and the least squares loss on a single trajectory are
\begin{equation}
   L_{\textup{data}}(\theta) = \frac12 \sum_{k=0}^{R-1}\left\| G_{\theta, h}(\hat{\mathbf u}_k) - \mathbf u_{k+1} \right\|^2,
   \quad 
   L_{\textup{LS}}(\theta)=\frac12\sum_{k=0}^{R-1}\left\| R_h( G_{\theta,h}(\hat{\mathbf u}_k), \hat{\mathbf u}_k) \right\|^2
\end{equation}
where we use $\hat{\mathbf u}_k$ to avoid backpropagation through the graph of $G_{\theta,h}^k$---the $k$-th power of the FNO and $R\in\mathbb N$ is a hyperparameter of the training setup, we use $R=10$, hence train for 10 timesteps between $t=0$ and $t=0.1$. The preconditioned least-squares loss is based on an approximate inverse of the Jacobian of $\phi(\cdot, \mathbf u_k)$ realized via geometric multigrid
\begin{equation}
   L_{\textup{PLS}}(\theta) = \frac12\sum_{k=0}^{R-1}\left\| P[R_h(G_{\theta,h}(\hat{\mathbf u}_k), \hat{\mathbf u}_k)] \right\|^2,
   \quad
   P\approx \left( (3\varepsilon^2 + \tau^{-1})M + A \right)^{-1}.
\end{equation}
For a detailed derivation and the additional baselines PI-DeepONet and PINO, see Appendix~\ref{appendix:details_allen_cahn}.

\begin{figure}[t]
  \centering
  \includegraphics[width=0.49\linewidth]{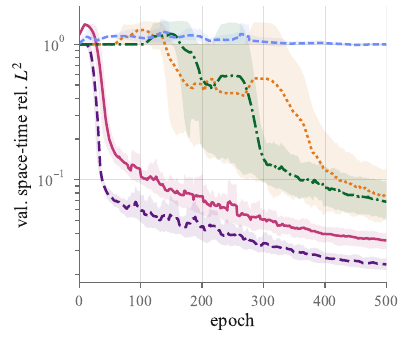}
  \hfill
  \includegraphics[width=0.49\linewidth]{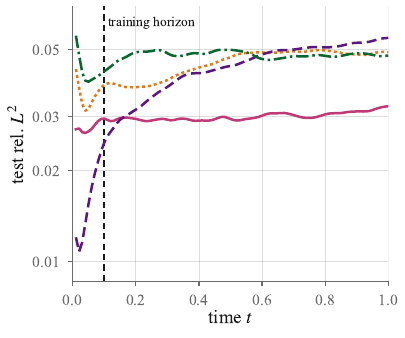}
  \par\vspace{2pt}
  \includegraphics[width=\linewidth]{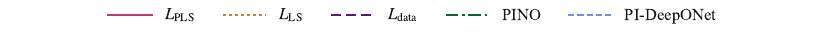}
  \caption{Allen--Cahn. \textbf{Left:} validation space-time relative $L^2$ (pooled over the validation set) vs.\ epoch; mean over three seeds with min--max band, smoothed by a centered 15-epoch moving average. \textbf{Right:} test relative $L^2$ per time step of a free roll-out to $T=1$, mean over test samples of the per-sample error, for the best seed of each loss (selected on validation); PI-DeepONet, whose error stays between $0.7$ and $1$, is omitted. The dashed line marks the training horizon $t=0.1$.}
  \label{fig:ac-results}
\end{figure}

\paragraph{Discussion} For this example, the results are qualitatively different: The PINO baseline and our un-preconditioned least squares baseline do produce workable predictions here although they are still trailing the data-driven and the preconditioned least-squares loss by a factor of at least two. This is not in contradiction to the remaining results, as time stepping operators are typically much tamer in their conditioning. We provide a detailed analysis in Appendix~\ref{appendix:details_allen_cahn}. Nevertheless, preconditioning still accelerates convergence and leads to an visible improvement in accuracy over the baselines and almost matches the data-driven benchmark. Interestingly, for roll-out beyond the training horizon all physics-informed losses are remarkably stable, see Figure~\ref{fig:ac-results}. We conjecture that this is due to the self-consistency of these losses---the next predicted step must be correct only with respect to the previous one, whereas data-driven training matches the whole trajectory at once.

\subsection{Stokes Equation}\label{sec:stokes}
We learn the solution operator $G: \rho \mapsto (u, p)$ of the stationary Stokes equations
\begin{align}\label{eq:stokes_equation}
   \begin{split}
      -\Delta u + \nabla p &= \rho, \qquad \nabla\cdot u = 0 \quad \text{in } \Omega,
      \\
      u &= 0 \quad \text{on } \partial\Omega, \qquad \textstyle\int_\Omega p\,\mathrm dx = 0,
   \end{split}
\end{align}
and the same family of sine-series body forces as in the Poisson experiment ($K = 10$). Stokes system is a saddle point problem and requires inf-sup stable pairs. We use Taylor--Hood $P_2/P_1$ and discretization yields an indefinite system matrix $A$ that leads to a condition number of $\mathcal O(h^{-4})$ in the least squares loss without preconditioning. Our domain is a unit square with a circular hole of radius $0.14$ centred at $(0.40, 0.50)$ discretized by an unstructured triangulation with $3998$ velocity and $1035$ pressure nodes ($9031$ degrees of freedom). As neural operator on the unstructured grid, we use GAOT \citep{gaot} and rely on algebraic multigrid \citep{naumov2015amgx} for the block-diagonal preconditioners described below. The full details and training protocol are documented in Appendix \ref{appendix:details_stokes}.

\paragraph{Block-diagonal Preconditioners} After discretization with the inf-sup stable Taylor-Hood pair, Stokes has the following block-structure
\begin{align*}
   A = 
   \begin{pmatrix}
      K & D^\top \\ 
      D & 0   
   \end{pmatrix},
   \quad
   \mathbf u(\theta, \rho) = 
   \begin{pmatrix}
      \mathbf w(\theta, \rho) \\
      \mathbf p(\theta, \rho)
   \end{pmatrix},
   \quad
   \mathbf f(\rho) =
   \begin{pmatrix}
      \mathbf g(\rho) \\
      0
   \end{pmatrix},
\end{align*}
where $K$ is symmetric and positive definite and $D^\top$ is injective, so that $A$ is symmetric but indefinite. The Schur complement is $S=DK^{-1}D^\top$, and the classical block-diagonal preconditioner is $P_S = \operatorname{diag}(K^{-1}, S^{-1})$; we approximate $K^{-1}$ by one algebraic multigrid V-cycle and $S^{-1}$ by the inverse of the lumped pressure mass matrix, $\hat S^{-1} = \omega_S\,\operatorname{diag}(M_p)^{-1}$. The scaling is dictated by the spectral equivalence $S \simeq M_p$, making $\omega_S>0$ a free parameter of $P_S$. We select it on the validation set and find a broad optimum over $\omega_S\in[4,16]$, see Appendix~\ref{appendix:details_stokes}. Its key property is that $P_SA$ has exactly three distinct eigenvalues $\{ 1, \frac{1+\sqrt 5}{2}, \frac{1-\sqrt 5}{2} \}$ \citep{murphy2000note}. In the Euclidean norm, however, $P_SA$ remains ill-conditioned, since it is not normal; in particular, $P_S$ is not an approximate inverse of $A$. We therefore use the pair $(P,B)=(I, P_S)$ in the preconditioned least-squares loss
\begin{align}
   \begin{split}
   L_{\textup{PLS}}(\theta)
   &=
   \tfrac12\,\mathbb E_{\rho\sim\mu}\big[
      \| A\mathbf u(\theta, \rho) - \mathbf f(\rho) \|_{P_S}^2
      \big]
   \end{split}
\end{align}
On a single sample, the Hessian of $L_{\textup{PLS}}$ with respect to $\mathbf u$ is $AP_SA$, whose condition number 
scales like $\kappa(P_S) = \mathcal O(h^{-2})$ versus $\mathcal O(h^{-4})$ of $L_{\textup{LS}}$ which is sufficient in our numerical experiments.

\paragraph{Discussion}
The Stokes experiment tests the method on an unstructured mesh where we replace the FNO by GAOT \citep{gaot} and use algebraic multigrid. Despite the $\mathcal O(h^{-2})$ conditioning that the block-diagonal preconditioner leaves in place, training is reliable and accurate: $L_\textup{PLS}$ reaches $1.6\%$ velocity and $0.85\%$ pressure error and thereby surpasses supervised training, at $3.3\%$ and $1.8\%$, by a factor of two on both fields, with a comparable cost per epoch. We have no clear explanation for this. It is, however, not excluded by construction: for Poisson, $L_\textup{PLS}$ is an approximation of the supervised loss, as $P\approx A^{-1}$, whereas the block-diagonal $P_S$ is not an approximate inverse. Without preconditioning, physics-informed training fails on this problem: the least-squares loss stalls at $30.6\%$ velocity and $11.4\%$ pressure error, and its error is smooth and large-scale---the low-frequency modes that converge last on such an ill-conditioned problem, see Figure~\ref{fig:stokes-samples}. PI-DeepONet, fails to train and ends at $39.7\%$ and $36.7\%$. Details are provided in Appendix~\ref{appendix:details_stokes}.

\begin{figure}[t]
  \centering
  \includegraphics[width=0.49\linewidth]{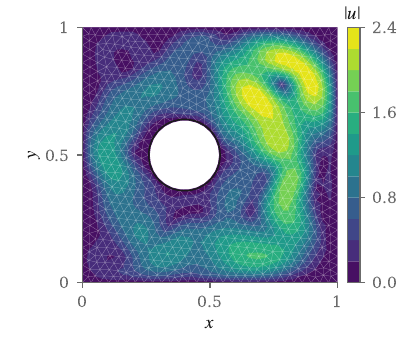}
  \includegraphics[width=0.49\linewidth]{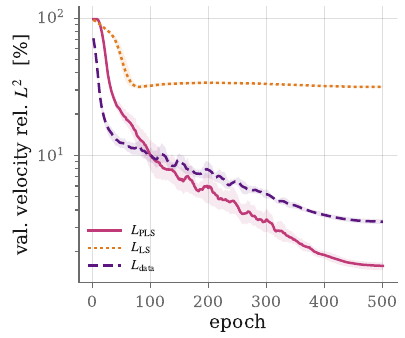}
  \caption{Stokes flow past an obstacle. \textbf{Left:} velocity magnitude of one reference solution to illustrate the domain, triangulation and zero Dirichlet (no-slip) boundary conditions. \textbf{Right:} relative $L^2$ error of the velocity on the validation set vs.\ epoch. We report the mean over three seeds and shade the min--max band. Preconditioned training separates from supervised training after roughly a hundred epochs and stays below it; the unpreconditioned control $L_\textup{LS}$ fails to converge and stagnates at around $30\%$ error.}
  \label{fig:stokes-results}
\end{figure}

\section{Conclusion and Limitations}
 Physics-informed operator learning has so far trailed its data-driven counterpart, largely due to the ill-conditioning induced by PDE operators. We show that classical preconditioners mitigate this when used inside a least-squares loss. This only changes the loss function, is agnostic to the neural operator architecture and incurs no overhead at inference. Across the Poisson, Allen–Cahn and Stokes equations, on structured and unstructured grids with FNO and GAOT backbones, the resulting label-free training matches or closely approaches supervised training. The main limitation is that the method relies on an efficient preconditioner for the problem class at hand; for problems where such preconditioners are not readily available, such as convection-dominated or high-frequency wave problems, this opens a direction for future research.

\subsection*{AI use statement}

In this work, we used generative AI tools for coding assistance, writing assistance, and mathematical reasoning assistance. We have reviewed all AI-assisted work. All LLM-generated output was verified and tested for correctness by the authors. We take responsibility for the final content of this work, including text, claims or artifacts produced with the aid of generative AI.

\subsection{Reproducibility}
The code to reproduce the experiments can be found at \url{https://github.com/camlab-ethz/TensorPILS}.

\bibliography{iclr2027_conference}
\bibliographystyle{iclr2027_conference}

\appendix

\section{Details on Interpolated Neural Operators}\label{appendix:interpolated_neural_operators_details}
\begin{wrapfigure}{r}{0.40\textwidth}
   \vspace{-\intextsep}
   \includegraphics[width=\linewidth]{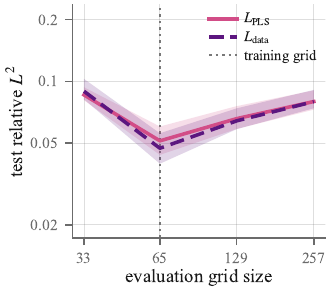}
   \vspace{-1em}
   \caption{Test error of the Poisson problem of Section~\ref{sec:poisson}, trained on the $65\times 65$ grid and evaluated on different resolutions. Mean over three seeds with min--max band. Both losses transfer in the same way across different grids.}
   \label{fig:resolution}
\end{wrapfigure}
Interpolation into a finite dimensional space $U_h$ via an interpolation operator $I_h$ may appear to tie neural operators to a fixed resolution, which would forfeit the discretization invariance that motivates operator learning in the first place \citep{kovachki2023neural}. This is not the case, as the interpolation can be understood as a part of the loss, not of the model. The trained neural operator $G_\theta$ may be evaluated at any resolution, with or without interpolation. However, the choice of $h$ does determine the accuracy the training can achieve, as features smaller than $h$ cannot be resolved. This is happens also in data-driven learning, where labels produced by a solver on the same mesh inherit the same limitation, together with the discretization error of the labels. Empirically, this is demonstrated in Figure~\ref{fig:resolution}. We give a theoretical discussion in the following, and to this end consider the three operators
\begin{align*}
   G &: \Pi \to U, \quad \rho \mapsto G(\rho),
   \\
   G_\theta &: \Pi \to U, \quad \rho \mapsto G_\theta(\rho),
   \\
   G_{\theta, h} &: \Pi \to U_h, \quad \rho \mapsto I_h\big( G_\theta(\rho) \big),
\end{align*}
where $G$ is the true underlying operator we are interested in, $G_\theta$ a neural surrogate with fixed $\theta$, and $G_{\theta,h}$ the interpolated neural operator. The operator $I_h$ maps into the finite-dimensional space $U_h$; it is typically defined on a subspace of $U$ only, as nodal approximation requires point evaluations, which are not available on $U=H^1(\Omega)$. A triangle inequality reveals the restriction on the accuracy obtainable by $G_{\theta,h}$: let $\rho\in\Pi$ be fixed, then
\begin{equation}\label{eq:appendix_interpolation_triangle}
   \big\| G(\rho) - G_{\theta,h}(\rho) \big\|_U
   \leq
   \underbrace{\big\| G(\rho) - G_\theta(\rho) \big\|_U}_{\text{network error}}
   +
   \underbrace{\big\| G_\theta(\rho) - G_{\theta,h}(\rho) \big\|_U}_{\text{interpolation error}}.
\end{equation}
If---as in our applications---$I_h$ is nodal interpolation into a $P_1$ space (or $Q_1$, $P_2$, \dots), the interpolation error in $L^2(\Omega)$ scales with $h^2$,
\begin{align*}
   \big\| G_\theta(\rho) - G_{\theta,h}(\rho) \big\|_{L^2(\Omega)}
   &\lesssim
   h^2\,\big| G_\theta(\rho) \big|_{H^2(\Omega)},
   \\
   \big\| G_\theta(\rho) - G_{\theta,h}(\rho) \big\|_{H^1(\Omega)}
   &\lesssim
   h\,\big| G_\theta(\rho) \big|_{H^2(\Omega)}.
\end{align*}

So for fixed $h$, we get an error floor in \eqref{eq:appendix_interpolation_triangle}. On our $65\times 65$ grids, this floor---at $h^2 = 64^{-2}\approx 2.4\cdot 10^{-4}$---is of the order $10^{-4}$ for $L^2$, well below all errors we measured.

Note that similar errors are present in data-driven neural operator training, where the data is produced with classical methods and frequently lies in a finite-dimensional space $U_h$ itself, which might even be the space used for physics-informed training. Figure~\ref{fig:resolution} confirms this empirically: evaluating the trained operators of Section~\ref{sec:poisson} on grids from $33\times 33$ to $257\times 257$, the preconditioned and the supervised arm degrade alike away from the training resolution, and neither is tied to it by more than the accuracy its training resolution could ask for.

\section{The Different Options of Preconditioning}
Besides the general applicability of least-squares based preconditioning---it doesn't require a variational formulation and works for both linear and nonlinear equations---we want to illustrate why it is our method of choice. To this end, we discuss the available options and compare them. As we aim to mitigate the ill-conditioning in the output layer of the neural operator, and the latter is nonlinear, not all preconditioning approaches of numerical linear algebra apply directly. The advantages of least squares based preconditioning are:
\begin{itemize}
   \item It incurs no cost at inference, as the preconditioner is part of the loss, compare to Appendix~\ref{appendix:preconditioning-through-architecture}.
   \item It poses no restriction on the choice of preconditioner, see Appendix~\ref{appendix:preconditioning-through-inner-product}
   \item It works for general PDEs. Identically to the other approaches, the choice of preconditioner remains PDE-dependent, summarized again in Appendix~\ref{appendix:least-squares-preconditioning}.
\end{itemize}

From here on, we specialize to the case of quadratic loss functions and assume that our operator learning setup is to learn a solution operator. On a single sample ($\mathbf f = \mathbf f(\rho) $), the least-squares and the variational loss then are 
\begin{align*}
   L_{\textup{LS}}(\theta) 
   &=
   \frac 12 \| A \mathbf u(\theta) - \mathbf f \|^2,
   \\
   L_{\textup{Var}}(\theta)
   &=
   \frac12 \mathbf u(\theta)^\top A \mathbf u(\theta) - \mathbf f^\top \mathbf u(\theta).
\end{align*}
In the case of $L_{\textup{Var}}$, we furthermore assume that $A$ is positive definite and symmetric. In the above, $\mathbf u(\theta)$ can be interpreted as an interpolated neural operator, see Appendix~\ref{appendix:interpolated_neural_operators_details}. We suppress its dependence on $\rho$ for brevity.

\subsection{Least Squares Preconditioning}\label{appendix:least-squares-preconditioning}
In this setting, the preconditioned least-squares loss is
\begin{equation*}
   L_{\textup{PLS}}(\theta)
   =
   \frac12\| P[A\mathbf u(\theta) - \mathbf f] \|_B^2,
\end{equation*}
where $P$ is bijective and linear, and $B$ is positive definite and symmetric, hence induces a norm. First order optimality conditions read
\begin{equation*}
   \nabla L_{\textup{PLS}}(\theta) = 0
   \quad \Leftrightarrow \quad
   J(\theta)^\top [A^\top P^\top B P A\, \mathbf u(\theta)] = J(\theta)^\top [A^\top P^\top B P \mathbf f],
\end{equation*}
where $J(\theta)$ is the Jacobian of $\theta \mapsto \mathbf u(\theta)$. Up to the Jacobian, these are the normal equations of the left preconditioned system $PA\mathbf u = P\mathbf f$ in the $B$-inner product. The Gauss-Newton matrix is
\begin{equation*}
   \operatorname{GN}[L_{\textup{PLS}}](\theta)
   =
   J(\theta)^\top [A^\top P^\top B P A] J(\theta).
\end{equation*}
The following Lemma quantifies the conditioning of the Gauss-Newton matrix up to the influence of the Jacobians.
\begin{lemma}\label{lemma:conditioning}
   Let $B\in\mathbb R^{n\times n}$ be positive definite and symmetric. Let $A, P\in\mathbb R^{n\times n}$ be invertible. Then we can estimate
   \begin{align}
      \kappa(A^\top P^\top B P A)
      &\leq
      \kappa(PA)^2\kappa(B)
      \label{eq:lemma_bound_least_squares}
      \\
      \intertext{and, if $A$ is in addition symmetric,}
      \kappa(ABA)
      &\leq
      \kappa(B)\left( \frac{\max |\lambda(BA)|}{\min|\lambda(BA)|} \right)^2.
      \label{eq:lemma_bound_norm_weight}
    \end{align}
\end{lemma}
\begin{proof}
    For \eqref{eq:lemma_bound_least_squares} we abbreviate $M = PA$. The condition number is submultiplicative and invariant under transposition, hence
    \begin{equation*}
        \kappa(A^\top P^\top B P A)
        =
        \kappa(M^\top B M)
        \leq
        \kappa(M^\top)\,\kappa(B)\,\kappa(M)
        =
        \kappa(PA)^2\kappa(B).
    \end{equation*}
    For \eqref{eq:lemma_bound_norm_weight} we define the matrix
    \begin{equation*}
        S = B^{1/2} A B^{1/2} = B^{-1/2}(BA)B^{1/2},
    \end{equation*}
    which is similar to $BA$ so the eigenvalues of $BA$ and $S$ are identical. As $S$ is symmetric, the condition number of $S$ is given through its eigenvalues
    \begin{equation*}
        \kappa(S) = \frac{\max|\lambda(S)|}{\min|\lambda(S)|} = \frac{\max |\lambda(BA)|}{\min|\lambda(BA)|}.
    \end{equation*}
   Since $ABA = B^{-1/2}S^2B^{-1/2}$ and $\kappa(B^{-1/2}) = \kappa(B)^{1/2}$, we conclude
    \begin{equation*}
        \kappa(ABA)
        =
        \kappa(B^{-1/2}S^2B^{-1/2})
        \leq
        \kappa(B)\,\kappa(S^2)
        =
        \kappa(B)\left( \frac{\max |\lambda(BA)|}{\min|\lambda(BA)|} \right)^2,
    \end{equation*}
    where $\kappa(S^2) = \kappa(S)^2$ holds because $S$ is symmetric.
\end{proof}

\subsection{Preconditioning Through Architecture Adaption}\label{appendix:preconditioning-through-architecture}
We now modify the neural network architecture via appending a non-trainable final preconditioning layer $\mathbf v(\theta) = P \mathbf u(\theta)$, where $P$ is bijective and linear. Note that this implies that $P$ \emph{must be applied at inference time}. The corresponding loss functions are
\begin{align*}
   \tilde L_{\textup{LS}}(\theta) 
   &=
   \frac 12 \| A \mathbf v(\theta) - \mathbf f \|^2
   =
   \frac 12 \| A  P \mathbf u(\theta) - \mathbf f \|^2
   ,
   \\
   \tilde L_{\textup{Var}}(\theta)
   &=
   \frac12 \mathbf v(\theta)^\top A \mathbf v(\theta) - \mathbf f^\top \mathbf v(\theta)
   =
   \frac12 \mathbf u(\theta)^\top P^\top A P \mathbf u(\theta) - \mathbf f^\top P \mathbf u(\theta)
   .
\end{align*}
First-order critical conditions read
\begin{align*}
   \begin{split}
      \nabla \tilde L_{\textup{LS}}(\theta) = 0 \quad  \Leftrightarrow \quad 
      &
      J(\theta)^\top [P^\top A^\top A P \mathbf u(\theta)] = J(\theta)^\top [P^\top A^\top \mathbf f],
      \\
      \nabla \tilde L_{\textup{Var}}(\theta) = 0 \quad  \Leftrightarrow \quad 
      &
      J(\theta)^\top [P^\top A P \mathbf u(\theta)] = J(\theta)^\top [P^\top \mathbf f],
   \end{split}
\end{align*}
where $J(\theta)$ denotes the Jacobian of $\theta \mapsto \mathbf u(\theta)$ at $\theta$. This shows that this approach is closest to two-sided preconditioning in numerical linear algebra. To assess the effect of $P$ on the ill-conditioning introduced through $A$, we compute the Gauss-Newton matrices
\begin{align*}
   \operatorname{GN}[ L_{\textup{LS}}](\theta)
   &=
   J(\theta)^\top [P^\top A^\top A P] J(\theta)
   \\
   \operatorname{GN}[ L_{\textup{Var}}](\theta)
   &=
   J(\theta)^\top[P^\top A P]J(\theta),
\end{align*}
Thus, ill-conditioning is mitigated through the choice $P\approx A^{-1}$ in the least squares case, and $P\approx A^{-1/2}$ for the variational formulation, or more generally, when the corresponding condition numbers are reduced.

The approach described above has been considered in the literature for PINN-type parametric learning problems, where parameter spaces are low dimensional.
\citet{bachmayr2026preconditioning} precondition through a modified neural network architecture and rely on the BPX preconditioner \citep{bramble1990parallel, brenner2008mathematical}, which is folded into their neural network ansatz. \citet{griese2025preconditioned} minimize a two-sided preconditioned residual for the stationary Stokes and Navier--Stokes equations, with the block factors obtained from Cholesky factorizations. 

\subsection{Preconditioning Through Change of Inner Product}\label{appendix:preconditioning-through-inner-product}
We modify $L_{\textup{Var}}$ by modifying the inner product
\begin{equation*}
   \hat L_{\textup{Var}}(\theta)
   =
   \frac12 (\mathbf u(\theta), A \mathbf u(\theta))_P - (\mathbf f, \mathbf u(\theta))_P
\end{equation*}
where $(\cdot, \cdot)_P=(\cdot, P\cdot)_{\mathbb R^n}$. We require both $A$ and $P$ to be symmetric and positive definite. Additionally, we need that $A$ and $P$ commute, which is necessary for the unique minimizers of $L_{\textup{Var}}$ and $\hat L_{\textup{Var}}$ to agree. Using $AP = PA$, the critical points satisfy
\begin{equation*}
   \nabla \hat L_{\textup{Var}}(\theta) = 0 
   \quad \Leftrightarrow \quad 
   J(\theta)^\top[PA\mathbf u(\theta)] = J(\theta)^\top [P \mathbf f]
\end{equation*}
which shows the similarity to left preconditioning in numerical linear algebra. The Gauss-Newton matrix is
\begin{equation*}
   \operatorname{GN}[\hat L_{\textup{Var}}](\theta)
   =
   J(\theta)^\top PA J(\theta)
\end{equation*}
and a good preconditioner reduces the condition number of $PA$, for instance $P\approx A^{-1}$. Note that the requirement $AP=PA$ is restrictive. For instance, it applies to preconditioners that are polynomial in $A$, but not to multigrid methods.

\subsection{Preconditioning in Parameter Space}
Finally, we consider preconditioning in parameter space, where we can assume a general setup. Consider minimization of a generic differentiable loss function $L:\Theta \to \mathbb R$. Using \emph{preconditioned gradient descent} to minimize $L$ is a general preconditioning strategy
\begin{equation*}
   \theta_{k+1} = \theta_k - \eta_k P_k \nabla L(\theta_k),
\end{equation*}
where $\eta_k>0$ is the learning rate and $P_k$ is the preconditioner, typically a linear and invertible matrix. Examples include Newton's method [nocedal], natural gradient descent [amari] and Adam [kingma] can be seen as a generalization of the above scheme. Preconditioning happens here in the parameter space $\Theta$ and requires specialized techniques to scale to large-scale networks [KFAC, SOAP, Shampoo]. We consider this approach orthogonal to the methods developed in this work.

\section{Details of the Poisson Equation Experiment}\label{appendix:details_poisson}

\paragraph{Operator} We consider the Poisson equation with homogeneous Dirichlet boundary conditions on the unit square $\Omega = [0,1]^2$,
\begin{equation}\label{eq:appendix_poisson}
   -\Delta u = \rho \quad \text{in } \Omega, \qquad u = 0 \quad \text{on } \partial\Omega.
\end{equation}
For every $\rho\in H^{1}_0(\Omega)^*$ the Poisson equation admits a unique weak solution $u\in H^1_0(\Omega)$. We learn the associated solution operator 
\begin{equation*}
   G: \Pi \to  U, \quad \rho \mapsto u,
\end{equation*}
with $U = V= H^1_0(\Omega)$ and $ \Pi = U^*$. Hence $ G$ is the inverse of the Laplace operator and the residual formulation of Section \ref{sec:physics_informed_training} is 
\begin{equation*}
   R:H_0^1(\Omega) \times H_0^1(\Omega)^* \to H^1_0(\Omega)^*,
   \quad
   R(u, \rho)(v) = \int_\Omega \nabla u \nabla v \,\mathrm dx - \rho(v).
\end{equation*}

\paragraph{Data Set} The dataset pairs $(\rho, u)$ are constructed through the eigenfunctions of the Laplacian, hence require no PDE solve. We draw coefficients $a_{ij}\sim\mathrm{Unif}[-1,1]$, independently for $i,j=1,\dots,K$, and set
\begin{align}\label{eq:appendix_poisson_data}
   \begin{split}
      \rho(x,y) &= \frac{\pi}{K^2}\sum_{i,j=1}^{K} a_{ij}\,(i^2+j^2)^{1/2}\,\sin(\pi i x)\sin(\pi j y),
      \\
      u(x,y) &= \frac{1}{\pi K^2}\sum_{i,j=1}^{K} a_{ij}\,(i^2+j^2)^{-1/2}\,\sin(\pi i x)\sin(\pi j y).
   \end{split}
\end{align}
The pairs satisfy $-\Delta u = \rho$ exactly and vanish at the boundary. The parameter $K$ controls the amount of frequencies in the data, and a larger $K$ constitutes a more challenging learning problem. Figure~\ref{fig:poisson_dataset} shows one draw at each value.

\begin{figure}[htb]
   \centering
   \includegraphics[width=\linewidth]{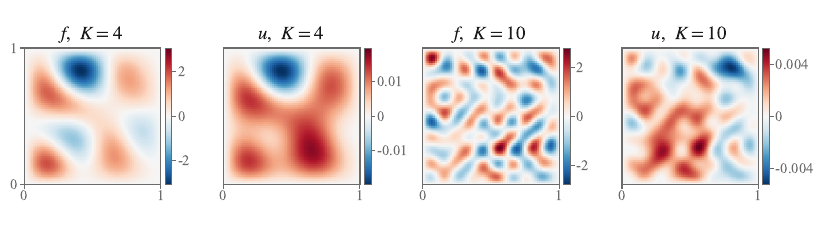}
   \caption{Samples from the Poisson data set, drawn from \eqref{eq:appendix_poisson_data} on a $65^2$ grid: source and solution at $K=4$ (left pair) and at $K=10$ (right pair). Note the differing magnitudes of $u$.}
   \label{fig:poisson_dataset}
\end{figure}

\paragraph{Discretization} We partition $\Omega$ into a uniform grid of $(n_x-1)^2$ squares with mesh size $h = 1/(n_x-1)$, and approximate in the space of continuous, piecewise \emph{bilinear} finite elements
\begin{equation*}
   U_h = Q_1^h = \big\{ v \in C(\overline\Omega) \;:\; v|_T \in \operatorname{span}\{1, x, y, xy\} \text{ for every cell } T, \quad v|_{\partial\Omega} = 0 \big\}.
\end{equation*}
We write $\phi_1, \dots, \phi_n$ for the associated global basis functions that are characterized through $\phi_i(x_j) = \delta_{ij}$ on the interior grid nodes $x_1, \dots, x_n$. This means that $n = \dim U_h = (n_x-2)^2$. We use $n_x = 65$, hence $n = 3969$. The stiffness and mass matrices are
\begin{equation}\label{eq:appendix_stiffness_mass}
   A_{ij} = \int_\Omega \nabla\phi_j \cdot \nabla\phi_i \,\mathrm dx,
   \qquad
   M_{ij} = \int_\Omega \phi_j\,\phi_i \,\mathrm dx,
\end{equation}
which we assemble with an exact quadrature rule. For the analytical sources $\rho$ we use its \emph{nodal approximation}, i.e., we replace $\rho$ by its interpolant $\sum_{j=1}^n \rho(x_j)\phi_j \in U_h$. The discretized forcing thus becomes
\begin{equation}\label{eq:appendix_consistent_rhs}
   \mathbf f_i = \mathbf f(\rho)_i = \sum_{j=1}^n M_{ij}\,\rho(x_j).
\end{equation}

\paragraph{Neural Operator} We use a FNO as implemented in the \texttt{neuraloperator} package \citep{kossaifi2026library}. The hyperparameters are collected in Table~\ref{tab:fno}. Because the FNO grid coincides with the finite element nodes, the interpolation into $ U_h$ of is the identity. Zero boundary conditions are hard-enforced by projecting to zero on $\partial\Omega$, so the model output lies in $U_h$ exactly and requires no boundary penalty.

\paragraph{PINO Baseline} PINO \citep{li2024physics} minimizes the \emph{strong-form} residual $-\Delta u - \rho$ at the interior grid nodes as the relative ratio $\|{-\Delta_h u} - \rho\|_2 / \|\rho\|_2$. It uses the FNO of Table~\ref{tab:fno} with second-order central differences on the grid. We impose the boundary condition exactly by setting the boundary nodes of the output to zero, exactly as we do it for $L_{\textup{LS}}$ and $L_{\textup{PLS}}$. The reference implementations instead multiplies the output by the ansatz $\sin(\pi x)\sin(\pi y)$, which on the present data set, a truncated sine series, encodes much of the structure of the solution and does not transfer to other geometries; we therefore report it only as a control. 

\paragraph{PI-DeepONet Baseline}
The other baseline, PI-DeepONet, \citep{wang2021learning} uses an unstacked DeepONet ($p=128$ basis functions, branch and trunk MLPs of width $256$ and depth $4$, $64$ random Fourier features on the trunk input, $1.4\cdot 10^6$ parameters) differentiated by automatic differentiation through the coordinate input of the trunk. PI-DeepONet uses a soft boundary penalty as in the original formulation with penalization strength $\lambda_{\mathrm{bc}}$. 
Its seed-to-seed spread, $0.163$ to $0.258$, is the largest in the table, and every baseline run is still descending at the end of the budget, so these numbers characterize the fixed budget rather than a converged state.

\begin{table}[htb]
   \caption{FNO hyperparameters used for all Poisson experiments.}
   \label{tab:fno}
   \centering
   \footnotesize
   \begin{tabular}{ll}
      \toprule
      \textbf{Component} & \textbf{Value} \\
      \midrule
      Fourier modes.             & $16\times16$ \\
      Hidden channels (width)    & $64$ \\
      Fourier layers             & $5$ \\
      Lifting / projection channels & $128$ / $128$ \\
      Positional embedding       & grid coordinates, appended as channels \\
      Input / output channels    & $1$ / $1$ \\
      Trainable parameters       & $3.01\times10^{6}$ \\
      \bottomrule
   \end{tabular}
\end{table}

\paragraph{Multigrid Preconditioner} We take $P$ to be a single geometric multigrid V-cycle for $A$ starting from a zero initial guess. We use dyadic coarsening on grids of fineness $2^k+1$ to guarantee nested grids, and denote by $A_\ell$ the stiffness matrix on level $\ell=0,\dots,L-1$, where $\ell=0$ is the finest grid and $L$ the number of levels. As smoother, we use weighted Jacobi with $\nu_1=\nu_2=2$ iterations and damping $\omega=8/9$, which is the optimal value for $Q_1$ elements in two dimensions, bilinear prolongation $\Pi_\ell$ with restriction $\Pi_\ell^\top$, and a direct solve on the coarsest grid; Algorithm~\ref{alg:vcycle} summarizes the V-cycle.

\begin{algorithm}[htb]
   \caption{V-cycle $\mathrm{V}(\ell,\mathbf r)$ on level $\ell$; the preconditioner is $P\mathbf r := \mathrm{V}(0,\mathbf r)$. Here $D_\ell = \operatorname{diag}(A_\ell)$.}
   \label{alg:vcycle}
   \begin{algorithmic}[1]
      \If{$\ell = L-1$}
         \State \Return $A_{L-1}^{-1}\mathbf r$ \Comment{coarsest grid: direct solve}
      \EndIf
      \State $\mathbf e \gets \mathbf 0$
      \For{$\nu_1$ sweeps}
         \State $\mathbf e \gets \mathbf e + \omega\, D_\ell^{-1}\big(\mathbf r - A_\ell \mathbf e\big)$ \Comment{pre-smoothing}
      \EndFor
      \State $\mathbf r_c \gets \Pi_\ell^\top\big(\mathbf r - A_\ell \mathbf e\big)$ \Comment{restrict the residual}
      \State $\mathbf e \gets \mathbf e + \Pi_\ell\, \mathrm{V}(\ell+1, \mathbf r_c)$ \Comment{coarse-grid correction}
      \For{$\nu_2$ sweeps}
         \State $\mathbf e \gets \mathbf e + \omega\, D_\ell^{-1}\big(\mathbf r - A_\ell \mathbf e\big)$ \Comment{post-smoothing}
      \EndFor
      \State \Return $\mathbf e$
   \end{algorithmic}
\end{algorithm}

\paragraph{Optimization} The FNO is trained with Adam and a cosine annealed learning rate; see Table~\ref{tab:optim}. Every setting is shared across losses and preconditioners \emph{except the learning rate}, which is tuned per loss. The rate is selected by a grid search over $\{10^{-4}, 3\cdot 10^{-4}, 10^{-3}, 3\cdot 10^{-3}, 10^{-2}, 3\cdot 10^{-2}\}$, run at the full $500$-epoch budget with a single seed, choosing the rate that attains the lowest validation relative $L^2$ over training. For PI-DeepONet, whose boundary penalty introduces a second hyperparameter, the weight $\lambda_{\mathrm{bc}}$ is tuned jointly with the learning rate on the grid $\{3\cdot 10^{-5}, 10^{-4}, 3\cdot 10^{-4}, 10^{-3}\} \times \{0.3, 1, 3\}$, extended to $\lambda_{\mathrm{bc}} \in \{0.1, 0.03, 0.01, 0.003, 0.001, 0\}$ at the best rates whenever the optimum fell on the edge of the grid; the selected pair $(10^{-4}, 0.03)$ is an interior optimum in both variables, with the validation error rising to $0.24$ at $\lambda_{\mathrm{bc}}=0.01$ and $0.30$ at $0.1$. The final runs then repeat the selected configuration over three seeds. Throughout, the checkpoint attaining the lowest validation error is restored after training, and the reported test error is evaluated at that checkpoint.

\begin{table}[htb]
   \caption{Optimization hyperparameters for the Poisson experiments. Shared across losses except
   the learning rate, which is swept per loss as described above.}
   \label{tab:optim}
   \centering
   \footnotesize
   \begin{tabular}{ll}
      \toprule
      \textbf{Setting} & \textbf{Value} \\
      \midrule
      Optimizer                 & Adam, $\beta_1 = 0.9$, $\beta_2 = 0.999$, $\varepsilon = 10^{-8}$ \\
      Weight decay              & $0$ \\
      Schedule                  & cosine, annealed from $\eta$ to $\eta/10$ over all epochs \\
      Learning rate $\eta$      & $3\cdot 10^{-4}$ for $L_\textup{data}$ and $L_\textup{PLS}$ \\
                                & $3\cdot 10^{-3}$ for $L_\textup{LS}$ and PINO \\
                                & $10^{-4}$ for PI-DeepONet, with boundary penalty weight $\lambda_{\mathrm{bc}} = 0.03$ \\
      Epochs                    & $500$ \\
      Batch size                & $32$ \\
      Train / validation / test samples & $1024$ / $128$ / $256$ \\
      Model selection           & lowest validation relative $L^2$, restored after training \\
      Random seeds              & $42, 43, 44$ (final runs); $42$ for the learning-rate sweep \\
      \bottomrule
   \end{tabular}
\end{table}

\section{Details of the Allen--Cahn Experiment}\label{appendix:details_allen_cahn}

\paragraph{Operator} We consider the Allen--Cahn equation with homogeneous Dirichlet boundary conditions on the unit square $\Omega = [0,1]^2$ and the time interval $I = [0,T]$
\begin{equation}\label{eq:appendix_allen_cahn}
   \partial_t u - \Delta u + \varepsilon^2\big(u^3 - u\big) = 0 \quad \text{in } I\times\Omega,
   \qquad
   u = 0 \quad \text{on } I\times\partial\Omega,
   \qquad
   u(0) = \rho \quad \text{in } \Omega.
\end{equation}
We use $\varepsilon=32$, and are interested in the solution operator $\rho\mapsto u$ taking an initial condition to the solution $u$. However, we do not learn the full solution at once, but instead focus on a time-stepping operator. To this end, let $S(t)\rho = u(t)$ denote the solution semigroup. We are interested in approximating $S(\tau)$ for a fixed $\tau$ and use $\tau=0.01$ in our experiments. Both for generating training data, and for the physics-informed losses we must commit to a time discretization, and we choose the classical convex-concave splitting scheme, which is first order and unconditionally stable \citep{eyre1998unconditionally}. Hence given $u_k \in L^2(\Omega)$, the next time step $u_{k+1}\in H^1_0(\Omega)$ solves
\begin{equation}\label{eq:appendix_cc_step}
   \frac1\tau\big(u_{k+1} - u_k, v\big) + \big(\nabla u_{k+1}, \nabla v\big) + \varepsilon^2\big(u_{k+1}^3 - u_k, v\big) = 0
   \quad \text{for all } v\in H^1_0(\Omega),
\end{equation}
where $(\cdot,\cdot)$ denotes the $L^2(\Omega)$ inner product. The operator we learn is the resulting time-stepping operator
\begin{equation*}
    G : L^2(\Omega) \to H^1_0(\Omega), \quad u_k \mapsto u_{k+1},
\end{equation*}
whose iterates $G^k(\rho)$ are an approximation of $S(k\tau)\rho$. 
The operator $G(u_k)$ is characterized as the root of the residual $R : H^1_0(\Omega)\times L^2(\Omega) \to H^{1}_0(\Omega)^*$
\begin{equation}\label{eq:appendix_cc_residual}
   R(w, u_k)(v)
   =
   \frac1\tau\big(w - u_k, v\big) + \big(\nabla w, \nabla v\big) + \varepsilon^2\big(w^3 - u_k, v\big)
   \quad \text{for all } v\in H^1_0(\Omega),
\end{equation}

\paragraph{Data Set} For the initial conditons, we use the same dataset as in the case of Poisson's equation, the frequency parameter is set to $K=4$. A sample together with predictions is shown in Figure~\ref{fig:ac-snapshots}.

\paragraph{Spatial Discretization} As for Poisson's equation, we discretize \eqref{eq:appendix_cc_step} with bilinear finite elements $U_h = Q_1^h$ on a uniform $129\times 129$ grid whose nodes coincide with the FNO output grid. With the stiffness matrix $A$ and the mass matrix $M$, this yields the discrete residual
\begin{equation*}
   R_h(\mathbf u_{k+1}, \mathbf u_k)
   =
   M\left[ \frac{\mathbf u_{k+1} - \mathbf u_k}{\tau} + \varepsilon^2\left( \mathbf u_{k+1}^3 - \mathbf u_k \right) \right]
   +
   A \mathbf u_{k+1},
\end{equation*}
where the cube $\mathbf u_{k+1}^3$ is taken entrywise. This is the residual $\phi$ of \eqref{eq:convex_concave_time_stepper} in the main text.

\paragraph{Neural Operator and Baselines} These are identical to the Poisson case, we use an FNO and the DeepONet described above.

\paragraph{Preconditioner}
For a fixed sample and a fixed time step $k$, the preconditioned least-squares loss is $\frac12\|P R_h(\mathbf w, \hat{\mathbf u}_k)\|^2$ with $\mathbf w = G_{\theta,h}(\hat{\mathbf u}_k)$. As $\hat{\mathbf u}_k$ is detached, its Gauss--Newton matrix in $\mathbf w$ is $J(\mathbf w)^\top P^\top P J(\mathbf w)$, where
\begin{equation*}
   J(\mathbf w)
   =
   \partial_{\mathbf w} R_h(\mathbf w, \hat{\mathbf u}_k)
   =
   \tau^{-1}M + A + 3\varepsilon^2 M \operatorname{diag}(\mathbf w^2)
\end{equation*}
is the Jacobian of the residual, and the relevant conditioning is that of $PJ(\mathbf w)$. The Jacobian depends on the state and is not symmetric. We therefore freeze it at $\mathbf w^2 = 1$, the value in the two phases $u = \pm 1$ that the Allen--Cahn dynamics drives the solution to, which gives the symmetric positive definite and state-independent matrix
\begin{equation*}
   J_0 = \big(\tau^{-1} + 3\varepsilon^2\big) M + A,
   \qquad
   P \approx J_0^{-1}.
\end{equation*}
In particular, the multigrid hierarchy for $J_0$ is built once, before training. The unpreconditioned least-squares loss sees $\kappa(J)^2$, and the following lemma shows that $J_0$ is much better conditioned than the Poisson stiffness matrix $A$. We can get a rough estimate of the condition number via

\begin{equation*}
   h^2(\tau^{-1} + 1) 
   \lesssim
   \tau^{-1} M + 3 \varepsilon^2 M \operatorname{diag}(u^2) + A 
   \lesssim
   h^2 (\tau^{-1} + 3\varepsilon^2) + 1,
\end{equation*}
where we used the $h^2$ scaling of the mass matrix in two spatial dimensions, and $h^2\lesssim A \lesssim 1$ for the stiffness matrix. This implies
\begin{equation*}
   \kappa(J_0) 
   \lesssim
   \frac{h^2(\tau^{-1} + 3\varepsilon^2) + 1}{h^2(\tau^{-1} + 1)}
   \leq
   1 + \tau\varepsilon^2\left( 3 + \frac{1}{h^2\varepsilon^2} \right). 
\end{equation*}
The choice $h\varepsilon=1/4$ is reasonable for Allen-Cahn in order to resolve interfaces of width $\varepsilon$ and used in all our experiments. Applying this yields the simpler bound
\begin{equation*}
   \kappa(J_0)
   \lesssim
   1 + 19\tau\varepsilon^2
   =
   1 + \frac{19}{16}\frac{\tau}{h^2}.
\end{equation*}
We see that conditioning with $\mathcal O(h^{-2})$ remains in theory. In practice it is strongly counteracted by the stepsize $\tau$. This is why ill-conditioning is much less pronounced for our Allen-Cahn time stepper example.

\begin{figure}[p]
   \centering
   \includegraphics[width=\linewidth]{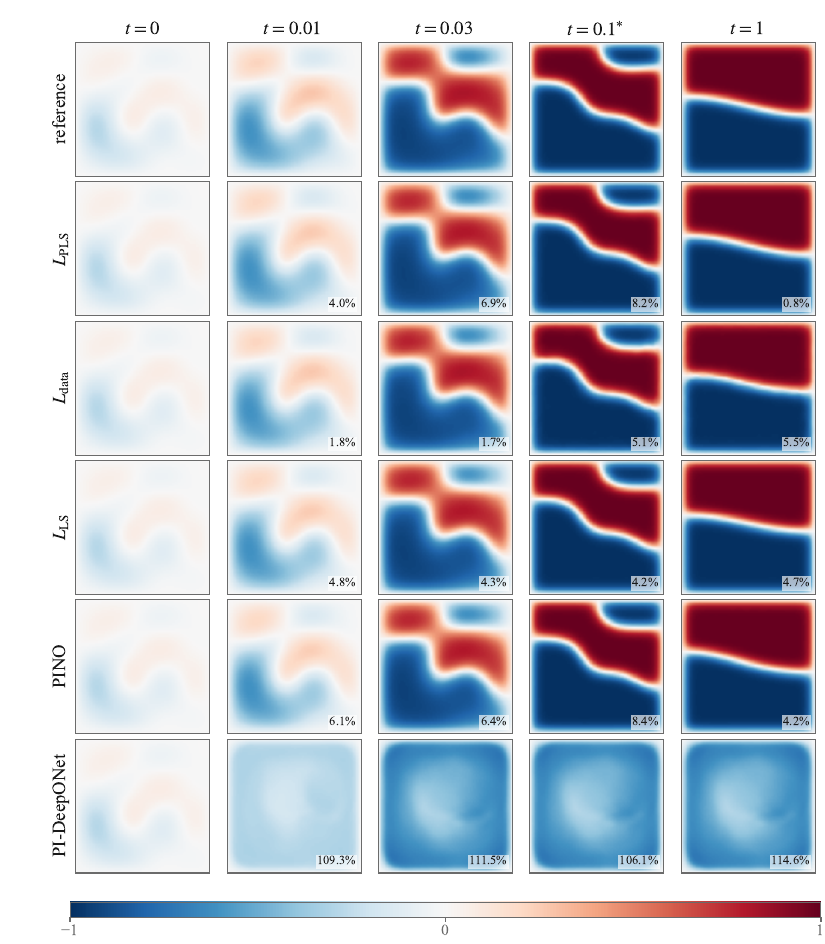}
   \caption{Roll-out of the Allen--Cahn time-stepping operator, from the initial condition to
   $T=1$ past the training horizon $t=0.1$ marked by the star. The different approaches are in the rows. Shown is one test sample. Between $t=0.01$ and $t=0.1$ the phases separate and the interfaces form, after which the solution only coarsens. All physics-informed arms except PI-DeepONet track the reference far beyond the training horizon, and the error of $L_\textup{PLS}$ decreases along the roll-out, to $0.8\%$ at $T=1$; PI-DeepONet never separates the phases.}
   \label{fig:ac-snapshots}
\end{figure}

\section{Details of the Stokes Equation Experiment}\label{appendix:details_stokes}

\paragraph{Operator} We consider the stationary Stokes equations \eqref{eq:stokes_equation} and learn the solution operator $G: \rho \mapsto (u, p)$ from the body force to the velocity--pressure pair. In the weak form, $ U = H^1_0(\Omega)^2 \times L^2_0(\Omega)$ and $a\big((u,p),(v,q)\big) = (\nabla u, \nabla v) - (p, \nabla\cdot v) - (q, \nabla\cdot u)$, a symmetric but indefinite bilinear form.

\paragraph{Domain and Mesh} $\Omega$ is the unit square with a circular hole of radius $0.14$ centred at $(0.40, 0.50)$. The hole is placed off-centre, so that the solution does not inherit symmetry from the domain, and the no-slip condition $u = 0$ is imposed on \emph{both} components of $\partial\Omega$. The mesh is an unstructured Gmsh triangulation at target edge length $0.035$, giving $1928$ elements, $3998$ velocity ($P_2$) nodes of which $284$ lie on $\partial\Omega$, and $1035$ pressure ($P_1$) nodes, i.e.\ $2\cdot 3998 + 1035 = 9031$ degrees of freedom with $568$ velocity components constrained. The size is chosen to match the $65^2 = 4225$ velocity and $33^2 = 1089$ pressure nodes of a structured Taylor--Hood grid at the same problem size, so that the dimension of the discrete problem and the cost of the model are comparable to the Poisson experiment. The boundary node set is taken from the incidence of the boundary facet elements.

\paragraph{Data Set} Each component of the body force is a truncated sine series with the same spectral decay as the Poisson sources,
\begin{equation}\label{eq:appendix_stokes_data}
   f_c(x,y) = \sum_{k,l=1}^{K} a_{ckl}\,(k^2+l^2)^{-1/2}\,\sin(k\pi x)\sin(l\pi y), \qquad c = 1, 2, \qquad a_{ckl} \sim \mathrm{Unif}[-1,1],
\end{equation}
with $K = 10$ as in the Poisson experiment. No closed-form solution exists; the reference $(u, p)$ is the \emph{discrete} Taylor--Hood solution of the system below, computed once per sample by a sparse direct solve in double precision with the pressure pinned at one node and subsequently re-gauged to zero mean, so that the label-free residual and the labels target the same discrete state. The factorization is shared across samples and the relative residual of the reference is $4.4\cdot 10^{-6}$. Stokes is linear, and we use the one free constant per sample to rescale $(f, u, p)$ so that the reference velocity has unit finite element $L^2$ norm. The physics then fixes the other two magnitudes, $u \sim f/ k^2$ and $p \sim f/k$: at unit velocity the nodal root-mean-square values are $\mathcal O(10^2)$ for $f$ and $\mathcal O(10)$ for $p$. The network therefore sees $f / f_{\mathrm{scale}}$, and its pressure channel is read as $p_{\mathrm{scale}}$ times the raw output; both are fixed constants of the data set applied at the network boundary, and neither the operator, the residual nor the reference sees them. Training, validation and test sets of $1024$, $128$ and $256$ samples are contiguous slices of one pool; as for Poisson, the seed draws both the pool and the initialization.

\paragraph{Discretization} We use the Taylor--Hood $P_2/P_1$ pair on the triangulation above: the velocity is continuous piecewise quadratic and the pressure continuous piecewise linear. Assembling the bilinear form gives the symmetric indefinite system
\begin{equation}\label{eq:appendix_stokes_system}
   A \mathbf u = \begin{pmatrix} K & D^\top \\ D & 0 \end{pmatrix} \begin{pmatrix} \mathbf w \\ \mathbf p \end{pmatrix} = \begin{pmatrix} M_u \mathbf g \\ 0 \end{pmatrix},
\end{equation}
where $K = \,\mathrm{diag}(K_{P_2}, K_{P_2})$ is the vector stiffness matrix, $D$ the discrete divergence, and $M_u \mathbf g$ the consistent load of the nodal interpolant of $g$. The residual is set to zero on the constrained velocity degrees of freedom, and the prediction is projected onto the admissible space: the velocity to zero on $\partial\Omega$ and the pressure to zero mean in the lumped $M_p$ inner product. 

\paragraph{Neural Operator} We use the Geometry-Aware Operator Transformer (GAOT) \citep{gaot}, with two input and three output channels $(u_x, u_y, p)$ and $3.40\cdot 10^{6}$ parameters. Since Stokes is the only experiment that changes backbone, we verify that the ordering of the losses is a property of the objective. We verify this by training both on the structured Poisson problem of Appendix~\ref{appendix:details_poisson}. The two architectures agree to within the seed spread, as reported in Table~\ref{tab:backbone}.

\begin{table}[htb]
   \caption{The ordering of the losses does not depend on the backbone. Structured Poisson,
   $64^2$, $K = 4$; test relative $L^2$ in percent, mean $\pm$ half-range over seeds $42/43/44$.
   Learning rate tuning picked $10^{-3}$ for both architectures.}
   \label{tab:backbone}
   \centering
   \footnotesize
   \begin{tabular}{lcc}
      \toprule
      loss & FNO & GAOT \\
      \midrule
      $L_\textup{data}$ (supervised)     & $0.59 \pm 0.06$  & $0.55 \pm 0.05$ \\
      $L_\textup{PLS}$ (preconditioned)  & $0.66 \pm 0.06$  & $0.61 \pm 0.07$ \\
      $L_\textup{LS}$ (bare residual)    & $29.13 \pm 2.22$ & $37.57 \pm 2.26$ \\
      \bottomrule
   \end{tabular}
\end{table}

\paragraph{Preconditioner} We use the block-diagonal $\mathcal P = \mathrm{diag}(K^{-1}, S^{-1})$ of Section~\ref{sec:stokes}. A geometric hierarchy does not exist on this mesh, so $A^{-1}$ is approximated by one \emph{algebraic} multigrid V$(2,2)$ cycle per velocity component; its measured contraction is about $0.4$ per cycle. $S^{-1}$ is approximated by the lumped pressure mass matrix, which is spectrally equivalent to the Schur complement $S$. The preconditioner enters the loss as residual weighting, for details see Appendix \ref{appendix:least-squares-preconditioning}.

\paragraph{Baseline} Of the tested methods in the literature, only PI-DeepONet transfers to the unstructured mesh. It minimizes the strong form of \eqref{eq:stokes_equation}. Its derivatives are taken by automatic differentiation through the coordinate input of the trunk, which is defined at any point of $\Omega$, whereas the finite-difference residual of PINO needs a grid and has no counterpart on an unstructured mesh. We use a DeepONet with three output channels, $p = 128$ basis functions, branch and trunk of width $256$ and depth $4$, $64$ random Fourier features on the trunk input and $2.54\cdot 10^{6}$ parameters, with the sensor set taken to be the mesh nodes. The boundary condition is imposed in a soft way with a penalty $\lambda_{\mathrm{bc}}$ as in the Poisson experiment. The residual is evaluated at all $3714$ interior nodes, so that the reported cost per epoch is comparable across the table. The requirement of second spatial derivatives is expensive in automatic differentiation: $6.0$ seconds per epoch against $1.5$ for the stochastic-collocation variant that draws $1024$ fresh nodes each step, for the same accuracy to within the seed spread.

\paragraph{Optimization} The optimization is carried out as in the Poisson experiment (Table~\ref{tab:optim}): Adam with a cosine schedule from $\eta$ to $\eta/10$ over $500$ epochs, batch size $32$, $1024$/$128$/$256$ samples, and for every method a grid search at the full budget on seed $42$ selecting the configuration with the lowest validation error, followed by three seeds. The validation error, and hence the model selection, is the mean of the velocity and pressure relative errors. The Schur weight was searched over $\omega_S \in \{0.5, 4, 16, 32, 64, 256\}$ at $\eta = 10^{-3}$, and for PI-DeepONet the learning rate, boundary weight and continuity weight over the grid of the structured Poisson study, whose selected values $(\eta, \lambda_{\mathrm{bc}}, w) = (10^{-4}, 0.1, 100)$ were re-selected here; at that rate $w = 1$ gives $80.9\%$ and $w = 1000$ gives $100.0\%$. Table~\ref{tab:optim_stokes} lists the selected values.

\begin{table}[htb]
   \caption{Hyperparameters for the Stokes experiment; everything not listed is as in Table~\ref{tab:optim}.}
   \label{tab:optim_stokes}
   \centering
   \footnotesize
   \begin{tabular}{ll}
      \toprule
      \textbf{Setting} & \textbf{Value} \\
      \midrule
      Learning rate $\eta$      & $10^{-3}$ for $L_\textup{PLS}$, with Schur weight $\omega_S = 16$ \\
                                & $10^{-3}$ for $L_\textup{data}$ and $L_\textup{LS}$ \\
                                & $10^{-4}$ for PI-DeepONet, with $w = 100$ and $\lambda_{\mathrm{bc}} = 0.1$ \\
      Collocation (PI-DeepONet) & all $3714$ interior nodes \\
      Model selection           & lowest validation mean of the velocity and pressure relative $L^2$ \\
      Preconditioner            & block; velocity half one algebraic V$(2,2)$ cycle \\
      Random seeds              & $42, 43, 44$ (final runs); $42$ for the grid search \\
      \bottomrule
   \end{tabular}
\end{table}

\paragraph{Metric} Velocity and pressure are scored in $L^2$ norm against the discrete reference, as the dataset-level relative error $\big(\sum_i \|e_i\|_M^2 / \sum_i \|c^*_i\|_M^2\big)^{1/2}$ with $M = M_u$ for the velocity and $M = M_p$ for the pressure, on the test set at the best-validation checkpoint; Table~\ref{tab:summary} reports both, each as mean $\pm$ half-range over the three seeds. The supervised and preconditioned runs are still improving at the end of the budget (best epoch $492$ and $498$ of $500$), so their numbers characterize the fixed budget rather than a converged state; the bare least-squares control is not, peaking between epochs $64$ and $77$ and drifting upwards thereafter.

\begin{figure}[htb]
  \centering
  \includegraphics[width=\linewidth]{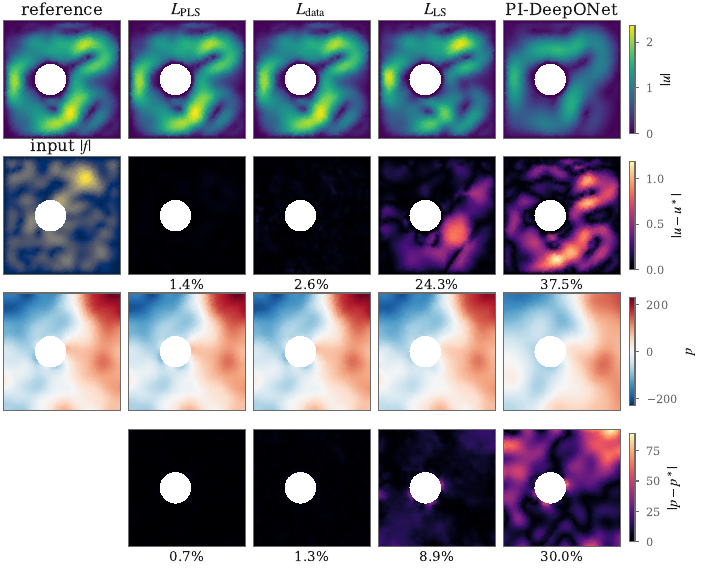}
  \caption{Rows, top to bottom: velocity magnitude; pointwise velocity error; pressure; pointwise pressure error. The leftmost column is the reference velocity, the body force that produced it, and the reference pressure.  $L_\textup{PLS}$ and $L_\textup{data}$ are visually indistinguishable from the reference in both fields; the unpreconditioned control has lost amplitude in the recirculation left of the obstacle and carries its pressure error in a ring around it; PI-DeepONet captures only the coarsest structure of either field.}
  \label{fig:stokes-samples}
\end{figure}

\paragraph{Preconditioner ablation} Table~\ref{tab:stokes_ablation} isolates the contribution of $\mathcal P$. Removing it entirely deteriorates the accuracy of the velocity by a factor of 20. Mis-weighting its pressure block is problematic too: $\omega_S$ multiplies the pressure block of the norm the loss measures in, so a large $\omega_S$ makes that norm pressure-dominated and optimization focuses mainly on pressure. At $\omega_S = 256$ the pressure error is still $3.9\%$ while the velocity error is $22.9\%$. The optimum is a broad plateau over $\omega_S \in [4, 16]$, which can also be seen numerically in the conditioning (Table~\ref{tab:stokes_ablation}); it is minimized at $\omega_S = 4$ and rises on both sides, in the same shape as the error, and the un-preconditioned residual's $\kappa(A^2) = 1.2\cdot 10^{11}$ is five orders of magnitude worse than the best preconditioned value. 
\begin{table}[htb]
   \caption{Preconditioner ablation on the unstructured Stokes problem: test relative finite
   element $L^2$ error, seed $42$, everything else as in Table~\ref{tab:summary}. The bold column is the selected value; selection is on the validation mean of the two fields, which is why it is $\omega_S = 16$ and
   not the marginally lower velocity error at $\omega_S = 4$. The last row is the condition number
   of the loss Hessian $\kappa(A^2)$ for $L_\textup{LS}$
   and $\kappa( A P A)$ for $L_\textup{PLS}$. It is minimized at $\omega_S = 4$ and follows the error across the sweep.}
   \label{tab:stokes_ablation}
   \centering
   \footnotesize
   \begin{tabular}{lccccccc}
      \toprule
      & $L_\textup{LS}$ & \multicolumn{6}{c}{$L_\textup{PLS}$, by Schur weight $\omega_S$} \\
      \cmidrule(lr){2-2} \cmidrule(lr){3-8}
      & no $\mathcal P$ & $0.5$ & $4$ & $16$ & $32$ & $64$ & $256$ \\
      \midrule
      velocity [\%] & $30.51$ & $2.58$ & $1.63$ & $\mathbf{1.69}$ & $3.35$ & $4.25$ & $22.94$ \\
      pressure [\%] & $10.22$ & $1.25$ & $0.95$ & $\mathbf{0.79}$ & $1.40$ & $1.62$ & $3.93$ \\
      \midrule
      $\kappa$ & $1.2\cdot 10^{11}$ & $3.7\cdot 10^{6}$ & $1.0\cdot 10^{6}$ & $1.9\cdot 10^{6}$
               & $3.2\cdot 10^{6}$ & $5.8\cdot 10^{6}$ & $2.1\cdot 10^{7}$ \\
      \bottomrule
   \end{tabular}
\end{table}

\end{document}